\documentclass{article}
\usepackage{iclr2026_conference,times}
\usepackage[T1]{fontenc}
\usepackage[utf8]{inputenc}
\usepackage{microtype}
\usepackage{inconsolata}
\usepackage{graphicx}
\usepackage{amsmath,amssymb,amsthm,mathtools}
\usepackage{booktabs}
\usepackage{multirow}
\usepackage{placeins}
\usepackage{hyperref}
\usepackage{url}

\usepackage{amsmath,amsfonts,bm}

\def\eqref#1{equation~\ref{#1}}

\def\1{\bm{1}}

\DeclareMathAlphabet{\mathsfit}{\encodingdefault}{\sfdefault}{m}{sl}
\SetMathAlphabet{\mathsfit}{bold}{\encodingdefault}{\sfdefault}{bx}{n}

\newtheorem{proposition}{Proposition}

\newtheorem{theorem}{Theorem}

\title{PRISM: Preference-aware Influence function based Data Selection Method for Efficient Fine-Tuning}

\author{
  Qihao Lin \quad
  Guanxu Chen \quad
  Dongrui Liu\thanks{Corresponding authors.} \quad
  Jing Shao\footnotemark[1] \\
  Shanghai Artificial Intelligence Laboratory \\
  \texttt{linqihao@pjlab.org.cn} \quad
  \texttt{shaojing@pjlab.org.cn}
}

\iclrfinalcopy

\begin{document}
\maketitle
\begin{abstract}
As LLMs continue to scale up, improving training efficiency heavily relies on effective data utilization. Data selection mitigates this issue by allocating the limited training budget to high-value examples that optimally facilitate the model’s target behavior. Most existing approaches define target behavior via a set of target examples and score candidate training data based on their estimated influence on these samples. However, such methods uniformly treat all target examples as equally important, ignoring the varying relevance of individual examples to model optimization. Specifically, target examples that align closely with the model’s inherent behavior deliver stronger supervisory signals, whereas discrepant examples yield only weak and ineffective local guidance. We propose PRISM, a \textbf{Pr}eference-aware \textbf{I}nfluence function based Data \textbf{S}election \textbf{M}ethod. It leverages model preference to assign weights to target examples and builds a preference-aware target direction. PRISM evaluates candidate training samples according to their influence on this direction, and prioritizes data budget allocation to samples that effectively drive the model to match expected target behavior. Theoretical analysis verifies that weighted preference construction generates a superior first-order gradient direction for boosting target preference, compared with uniform aggregation strategies. Extensive experiments covering diverse model architectures and parameter scales demonstrate that PRISM achieves better performance in efficient fine-tuning and safety-aligned supervised fine-tuning rectification. The results validate that accurate characterization of target behavior serves as the core of cost-effective data selection.
\end{abstract}

\section{Introduction}

As LLMs continue to scale, the costs of training and adaptation are steadily increasing, making training efficiency a critical challenge \citep{brown2020language,bommasani2021opportunities}. Beyond improving training algorithms \citep{hu2022lora,NEURIPS2023_1feb8787}, an important complementary direction under limited budgets is enhancing data efficiency by selectively training on the most valuable examples. Data selection provides a principled approach to this problem by identifying the training examples that contribute most to the target behavior \citep{xia2024less}.

\begin{figure}[t]
    \centering
    \includegraphics[width=\linewidth]{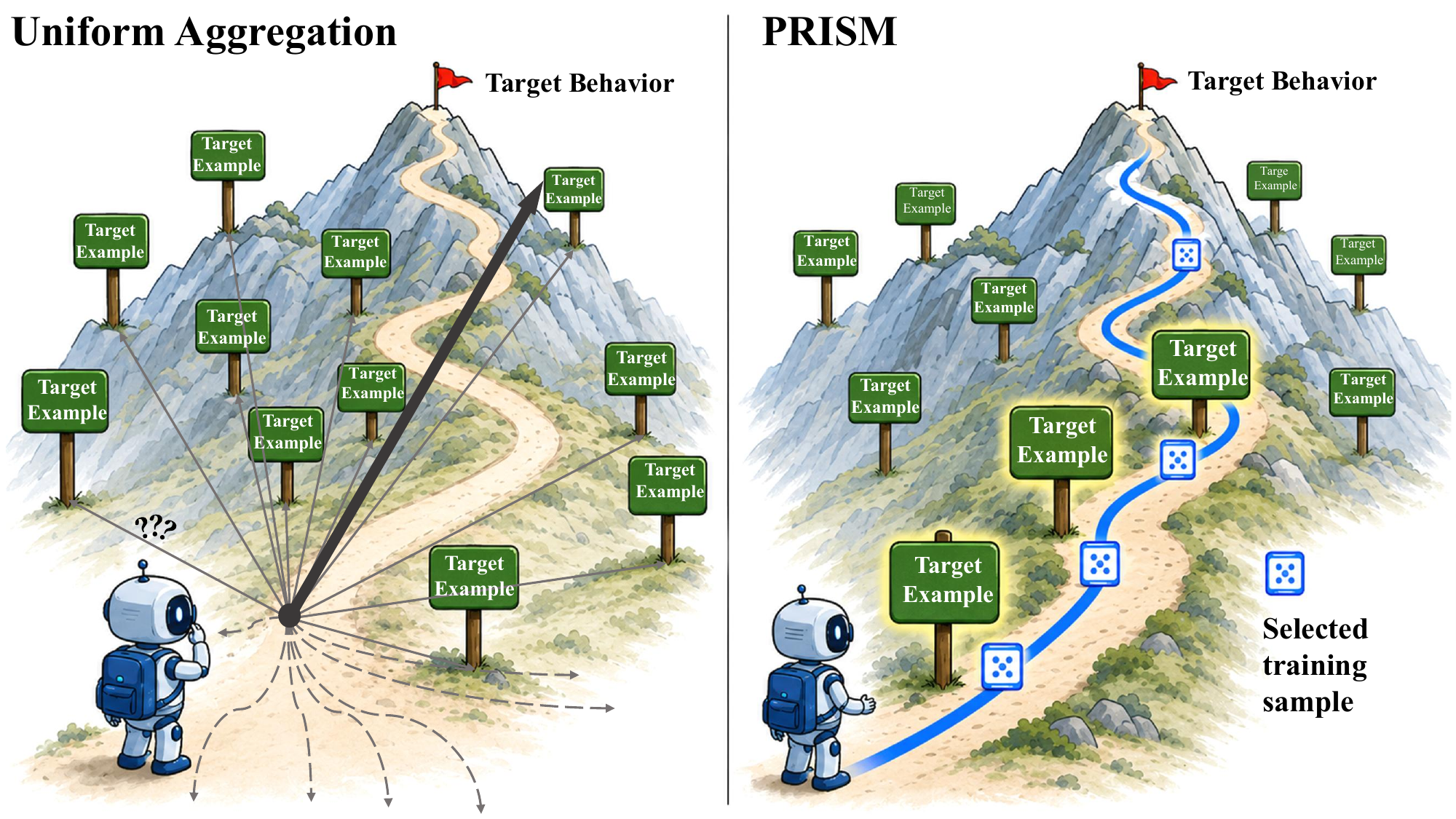}
    \caption{Motivation of PRISM. Equal aggregation treats all target examples equally, whereas PRISM emphasizes target examples closer to the model and focuses the budget on more actionable training examples.}
    \label{fig:motivation}
\vspace{-0.3cm}
\end{figure}

Data selection generally comprises two complementary steps: defining the target behavior that the model is expected to learn, typically represented by a set of target examples, and estimating the influence of each training example on that behavior. Although prior work has developed various methods to estimate influence more accurately or efficiently \citep{koh2017understanding,yeh2018representer,NEURIPS2020_e6385d39,kwon2024datainf}, these studies pay limited attention to precisely characterizing the target behavior itself. In particular, they often aggregate all target examples equally, ignoring their varying relevance to the model \citep{xia2024less,kowal2026concept,chen2025personaVectors,NEURIPS2025_715dbde2}. As illustrated in Figure~\ref{fig:motivation}, such equal aggregation blends diverse signals and may yield a target direction that fails to emphasize behaviors most likely to be amplified near the model.

To address this limitation, we account for differences in the importance of target examples when constructing the target behavior and propose PRISM (\textbf{Pr}eference-aware \textbf{I}nfluence function based Data \textbf{S}election \textbf{M}ethod). PRISM quantifies the model's preference for each target example using the likelihood assigned to the response and uses this preference to weight the example when constructing the target direction. A higher preference indicates that the example corresponds to a behavioral direction closer to the model's behavior and can therefore be strengthened more readily through training. Training examples associated with high-preference target examples are thus more likely to move the model efficiently toward the target behavior under a limited data budget. In contrast, lower-preference examples provide weaker selection signals. PRISM then scores candidate training examples by estimating their influence on this preference-aware target direction, thereby allocating the data budget to examples that are more likely to strengthen the target behavior efficiently.

Theoretically, we show that PRISM captures the first-order direction that most effectively increases target-behavior preference in the neighborhood of the model. Under a fixed selection budget, selecting examples along this direction maximizes the desired gain in target preference per unit of data. By contrast, equal aggregation mixes high- and low-preference target examples, causing the estimated direction to deviate from the behavioral pattern toward which the model can most readily be pushed, thereby reducing the utility of each selected example.

We conduct systematic experiments across model families and scales in two complementary data selection settings. First, to improve general capabilities, we define the target behavior in terms of factual knowledge, challenging reasoning, multilingual question answering, mathematical problem solving, and safety guardrail capabilities. Results show that, under the same data budget, training examples selected by PRISM lead to stronger fine-tuning performance. Second, in the safety-critical SFT repair setting for emergent misalignment \citep{betley2025emergent}, PRISM mitigates harmful model behaviors more effectively under the same removal budget. Overall, these findings demonstrate that PRISM improves the efficiency of data selection by providing a more precise characterization of the target behavior.

\section{Methodology}
\label{sec:method}

We now introduce PRISM (\textbf{Pr}eference-aware \textbf{I}nfluence function based Data \textbf{S}election \textbf{M}ethod), which improves influence-based data selection by constructing the target direction in a preference-aware way. Instead of treating all target examples equally, PRISM weights them by the model's preference and then traces the resulting target direction back to candidate training examples through influence scoring. Section~\ref{sec:method-preliminaries} introduces the notation, Section~\ref{sec:method-framework} derives the PRISM score, and Section~\ref{sec:theory} provides the theoretical analysis.

\subsection{Preliminaries}
\label{sec:method-preliminaries}

\noindent\textbf{Problem setup.} Let \(\mathcal{D}=\{z_i\}_{i=1}^n\) denote a candidate training pool, where each example \(z_i=(x_i,y_i)\) contains a query \(x_i\) and a training response \(y_i\). Let \(p_\theta\) denote the model distribution parameterized by \(\theta\), before training on the selected data. A target behavior is specified by a set of target queries \(\mathcal Q\). For each \(q\in\mathcal Q\), we use a paired target example \((q,y_q^+,y_q^-)\): \(y_q^+\) is the target-positive response whose behavior we want to align, and \(y_q^-\) is a target-negative response used as the contrast under the same query.

For any query-response pair \((u,v)\), define the sequence probability as
{\small
\begin{equation}
p_\theta(v\mid u)
=
\prod_{t=1}^{|v|}
p_\theta(v_t\mid u,v_{<t}).
\label{eq:method-prob}
\end{equation}
}
To avoid length bias when comparing target responses of different lengths, PRISM uses the length-normalized likelihood
{\small
\begin{equation}
\bar p_\theta(v\mid u)
=
p_\theta(v\mid u)^{1/|v|}.
\label{eq:method-normalized-prob}
\end{equation}
}
The corresponding token-average cross-entropy loss is
{\small
\begin{equation}
\bar\ell(u,v;\theta)
=
-\log \bar p_\theta(v\mid u)
=
\frac{1}{|v|}\ell(u,v;\theta).
\label{eq:method-normalized-loss}
\end{equation}
}
At the model \(\theta\), target-response gradients use the token-average loss, while candidate-training gradients follow the SFT objective:
{\small
\begin{equation}
\begin{aligned}
\bar g_{q,y_q}
=
\nabla_{\vartheta}\bar\ell(q,y_q;\vartheta)\big|_{\vartheta=\theta},
g_{z_i}
=
\nabla_{\vartheta}\ell(x_i,y_i;\vartheta)\big|_{\vartheta=\theta}.
\end{aligned}
\label{eq:method-gradients}
\end{equation}
}
Here \(\bar g_{q,y_q^+}\) and \(\bar g_{q,y_q^-}\) are the length-normalized gradients for generating the positive and negative target responses from \(q\), while \(g_{z_i}\) is the gradient for generating the training response \(y_i\) from \(x_i\).
\begin{figure*}[t]
\centering
\includegraphics[width=0.96\textwidth]{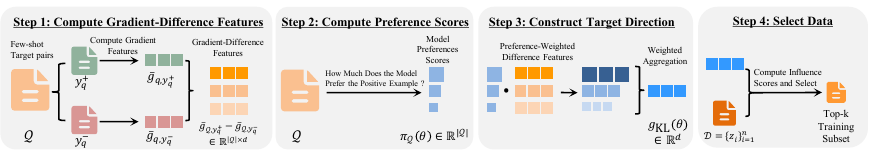}
\caption{Overview of PRISM. Given paired target examples, PRISM estimates the model's preference for the target-positive response, constructs a preference-aware target direction, and scores candidate training examples by influence alignment with this direction.}

\label{fig:prism-framework}
\vspace{-0.3cm}
\end{figure*}
\noindent\textbf{Local curvature.} We use \(H\) to denote the empirical Hessian of the training objective over the training distribution:
{\small
\begin{equation}
H=
\frac{1}{n}
\sum_{j=1}^{n}
\nabla_\theta^2 \ell(x_j,y_j;\theta).
\end{equation}
}
For the local influence analysis, we assume \(H\) is positive definite so that \(H^{-1}\) is well defined.

\noindent\textbf{Influence-based selection form.} At a high level, influence-based data selection scores each candidate training example by how well its training signal aligns with a target-behavior direction \citep{koh2017understanding,xia2024less,NEURIPS2025_715dbde2}. We denote this parameter-space representation of the target behavior by \(g_{\mathrm{target}}\). This common scoring form can be written as
{\small
\begin{equation}
s(z_i)
=
g_{z_i}^{\top}H^{-1}g_{\mathrm{target}},
\label{eq:method-general-influence}
\end{equation}
}
where \(g_{\mathrm{target}}\) represents the target behavior. Under this view, the key design question is how to construct \(g_{\mathrm{target}}\). Existing influence-based methods often set \(g_{\mathrm{target}}\) by equally aggregating gradients from target examples \citep{xia2024less,NEURIPS2025_715dbde2}. PRISM instead derives \(g_{\mathrm{target}}\) from a target-preference reward, yielding the preference-aware direction \(g_{\mathrm{KL}}\).

\subsection{PRISM Framework}
\label{sec:method-framework}

We now derive the PRISM score by constructing a preference-aware target direction and tracing it back to candidate training examples. Figure~\ref{fig:prism-framework} illustrates the resulting workflow.

\noindent\textbf{Model preference.} PRISM uses the model's relative length-normalized likelihood of the target-positive response as a proxy for how close this target example is to the model's behavior. For each paired target example, we define
{\small
\begin{equation}
\pi_q
=
\frac{
\bar p_{\theta}\!\left(y_q^+ \mid q\right)
}{
\bar p_{\theta}\!\left(y_q^+ \mid q\right)
+
\bar p_{\theta}\!\left(y_q^- \mid q\right)
}.
\label{eq:method-preference}
\end{equation}
}
This form is a two-response normalization over geometric-mean likelihoods: \(\pi_q\) is large when the model assigns higher length-normalized likelihood to the target-positive response than to the target-negative contrast. Thus, \(\pi_q\) measures the model's preference for the target-positive response within that pair. A larger \(\pi_q\) indicates that the target-positive response is closer to the model's generation tendency and should provide a more actionable local signal for data selection.

\noindent\textbf{Target-preference reward.} PRISM aims to improve the model’s preference for the target-positive response over the target-negative contrast. To make this objective explicit, we define a binary preference distribution for each target query,
{\small
\begin{equation}
P_q^\theta
=
\bigl(\pi_q(\theta),1-\pi_q(\theta)\bigr),
\label{eq:theory-P}
\end{equation}
}
where the first coordinate corresponds to the target-positive response. Let \(A_q=(0,1)\) denote the degenerate distribution that always selects the target-negative response. To turn the desired shift toward the positive response into an optimizable quantity, we measure how far the model's preference distribution moves away from this always-negative distribution. We define the target-preference reward as the KL divergence from \(A_q\) to \(P_q^\theta\):
{\small
\begin{equation}
\begin{aligned}
\mathcal K(\theta)
&=
\frac{1}{|\mathcal{Q}|}
\sum_{q\in\mathcal{Q}}
D_{\mathrm{KL}}(A_q\|P_q^\theta)\\
&=
\frac{1}{|\mathcal{Q}|}
\sum_{q\in\mathcal{Q}}
-\log\!\left(1-\pi_q(\theta)\right).
\end{aligned}
\label{eq:theory-K}
\end{equation}
}
This reward is small when the model stays close to the target-negative response and becomes larger when the model favors the target-positive response. It therefore directly measures how strongly the model tends toward the target behavior described by the paired examples, increasing this reward means making the model more likely to produce the desired behavior described by the paired examples. Appendix~\ref{app:proof-kl-basic} gives the closed form and related bounds.

\noindent\textbf{Preference-aware target direction.} Differentiating the target-preference reward rewrites this objective as a parameter-space direction. Because the length-normalized response likelihood is represented through the token-average loss in Eq.~\eqref{eq:method-normalized-loss}, we obtain
{\small
\begin{equation}
\begin{aligned}
-\nabla_\theta\mathcal K(\theta)
&=
\frac{1}{|\mathcal{Q}|}
\sum_{q\in\mathcal{Q}}
\pi_q(\theta)
\Bigl(
\bar g_{q,y_q^+}
-
\bar g_{q,y_q^-}
\Bigr)\\
&\triangleq
g_{\mathrm{KL}}(\theta).
\end{aligned}
\label{eq:method-gkl}
\end{equation}
}
Equation~\eqref{eq:method-gkl} is the preference-aware target direction traced by PRISM. The coefficient \(\pi_q\) is therefore not an ad hoc weight: it is induced by the gradient of the target-preference reward. Target examples with higher model preference contribute more to the direction because changing the model around those examples yields larger local gain in the reward. Appendix~\ref{app:proof-grad-kl} gives the derivation.

\noindent\textbf{Influence scoring.} We next trace the preference-aware target direction back to candidate training examples. Following influence function based data attribution \citep{koh2017understanding,xia2024less}, each candidate example \(z_i\in\mathcal D\) is scored by its Hessian-preconditioned alignment with the target direction:
{\small
\begin{equation}
s_{\pi}(z_i)
=
g_{z_i}^\top H^{-1}g_{\mathrm{KL}}.
\label{eq:method-score}
\end{equation}
}
A larger score means that the training signal of \(z_i\) is more aligned with the preference-aware target behavior under the local curvature of the training objective. PRISM ranks candidate examples by \(s_\pi\).
\begin{table}[t]
\centering
\scriptsize
\resizebox{\linewidth}{!}{
\begin{tabular}{lll}
\toprule
Method & Target direction & Score for \(z_i\) \\
\midrule
Equal aggregation
&
\(\frac{1}{|\mathcal Q|}\sum_{q\in\mathcal Q}(\bar g_{q,y_q^+}-\bar g_{q,y_q^-})\)
&
\(g_{z_i}^{\top}H^{-1}g_0\)
\\
PRISM
&
\(\frac{1}{|\mathcal Q|}\sum_{q\in\mathcal Q}\pi_q(\bar g_{q,y_q^+}-\bar g_{q,y_q^-})\)
&
\(g_{z_i}^{\top}H^{-1}g_{\mathrm{KL}}\)
\\
\bottomrule
\end{tabular}}
\caption{Comparison between equal aggregation and PRISM under the same influence-scoring form. PRISM changes the target direction by deriving preference weights from the target-preference reward.}
\label{tab:method-comparison}
\vspace{-0.3cm}
\end{table}

\subsection{Theoretical Analysis of PRISM}
\label{sec:theory}

Since \(g_{\mathrm{KL}}\) is the negative gradient of the target-preference reward, the remaining question is whether tracing this direction back to data yields an efficient selector. Theoretically, we proceed in two steps. First, Theorem~\ref{thm:influence-gain} shows that the PRISM score estimates how much each training example improves the target-preference reward. Second, Theorem~\ref{thm:equal-gap} compares PRISM with equal aggregation: the preference-aware direction gives at least as large a first-order reward gain under the same local budget, and its top-ranked subset yields at least as large a first-order gain under the same selection budget. The analysis is local around the model \(\theta\) and uses the positive-definite Hessian \(H\) defined above. We state the main results here and defer detailed derivations to Appendix~\ref{app:theory-proofs}.

\noindent\textbf{Influence score as reward gain.} We first justify the score used for ranking training examples. Theorem~\ref{thm:influence-gain} shows that \(s_\pi(z_i)\) is not merely an alignment score: it is the first-order change in the target-preference reward caused by infinitesimally upweighting \(z_i\).
\begin{theorem}[Influence score estimates target-preference gain]
\label{thm:influence-gain}
For any candidate example \(z_i\), let \(\theta_{\epsilon,i}\) be the local optimum after infinitesimally upweighting \(z_i\). Then the first-order change in the paired KL reward is
{\small
\begin{equation}
\frac{d}{d\epsilon}
\mathcal K(\theta_{\epsilon,i})\Big|_{\epsilon=0}
=
 g_{z_i}(\theta)^\top H^{-1}g_{\mathrm{KL}}(\theta)
\triangleq
h_\pi(z_i;\theta).
\label{eq:theory-hpi}
\end{equation}
}
For a size-\(m\) subset \(S\), the first-order effect is
{\small
\begin{equation}
\frac{d}{d\epsilon}\mathcal K(\theta_{S,\epsilon})\Big|_{\epsilon=0}
=
\frac{1}{m}\sum_{z_i\in S}h_\pi(z_i;\theta).
\label{eq:theory-remove-derivative}
\end{equation}
}
Therefore, selecting the top-\(m\) examples by \(h_\pi\) maximizes the first-order paired KL reward gain under the fixed budget.
\end{theorem}

The proof follows the standard influence functions differentiation of an infinitesimally upweighted training objective and is given in Appendix~\ref{app:proof-influence-gain}. Theorem~\ref{thm:influence-gain} gives the operational meaning of the PRISM score. The Hessian-preconditioned alignment \(g_{z_i}^\top H^{-1}g_{\mathrm{KL}}\) is not merely a similarity score; it estimates the first-order gain in target preference caused by training on \(z_i\).

We next analyze what changes when the target direction is constructed by equal aggregation. The comparison has two parts: equal aggregation may construct a target direction that deviates from the reward-maximizing preference-aware direction, and ranking examples with this deviated direction may produce a lower-gain top-\(m\) subset under the same reward. Equal aggregation removes the preference coefficients and uses

{\small
\begin{equation}
\begin{aligned}
g_0(\theta)
&=
\frac{1}{|\mathcal Q|}
\sum_{q\in\mathcal Q}
\Bigl(
 \bar g_{q,y_q^+}
 -
 \bar g_{q,y_q^-}
\Bigr),\\
h_0(z_i;\theta)
&=
g_{z_i}(\theta)^\top H^{-1}g_0(\theta).
\end{aligned}
\label{eq:theory-h0}
\end{equation}
}
\begin{table*}[!t]
\centering
\scriptsize
\setlength{\tabcolsep}{2pt}
\resizebox{\linewidth}{!}{
\begin{tabular}{lcccccccccccc}
\toprule
Method & & MMLU & BBH & TyDi F1 & TyDi EM & MATH-500
&  & MMLU & BBH & TyDi F1 & TyDi EM & MATH-500 \\
\midrule
\textsc{Random} & \multirow{7}{*}{\rotatebox[origin=c]{90}{Llama-2-7B}}
& 45.74 & 38.06 & 46.93 & 31.44 & 4.67
& \multirow{7}{*}{\rotatebox[origin=c]{90}{Llama-2-13B}} & 52.79 & 49.09 & 47.39 & 30.68 & 5.08 \\
\textsc{LESS} &
& \textbf{47.07} & \textbf{41.70} & 52.76 & 34.15 & 5.58
& & 54.25 & \textbf{50.09} & 52.12 & 34.98 & 6.17 \\
\textsc{VecFilter} &
& 45.10 & 37.35 & 46.62 & 31.95 & 5.08
& & 52.33 & 49.04 & 49.07 & 31.91 & 5.83 \\
\textsc{ProjDiff} &
& 44.70 & 39.69 & 47.26 & 31.30 & 4.83
& & 52.88 & 39.69 & 47.26 & 31.30 & 4.83 \\
\textsc{ConceptInf} &
& 44.15 & 37.87 & 47.43 & 28.91 & 2.75
& & 52.98 & 49.04 & 42.32 & 24.30 & 4.42 \\
\textbf{IF-Guide} &
& 47.00 & 37.31 & 52.99 & 36.39 & \textbf{6.33}
& & 54.24 & 49.17 & 51.44 & 34.74 & 6.00 \\
\textsc{PRISM} &
& 45.43 & 38.12 & \textbf{54.23} & \textbf{38.93} & \textbf{6.33}
& & \textbf{54.38} & 49.75 & \textbf{56.60} & \textbf{41.76} & \textbf{6.25} \\
\midrule
\textsc{Random} & \multirow{7}{*}{\rotatebox[origin=c]{90}{Qwen3-8B}}
& 75.30 & 73.80 & 69.81 & 49.74 & 22.50
& \multirow{7}{*}{\rotatebox[origin=c]{90}{Llama-3.1-8B-Inst.}} & 66.38 & 69.35 & 66.28 & 44.30 & 12.50 \\
\textsc{LESS} &
& 74.69 & 74.17 & 70.15 & 48.75 & 23.83
& & 66.93 & 68.02 & 68.27 & 48.06 & 13.83 \\
\textsc{VecFilter} &
& \textbf{75.32} & 73.95 & 68.75 & 47.46 & \textbf{24.83}
& & 65.30 & 68.49 & 67.80 & 45.65 & 11.08 \\
\textsc{ProjDiff} &
& 75.11 & 75.37 & 69.34 & 49.77 & 23.33
& & 66.59 & \textbf{70.37} & 65.41 & 43.61 & 10.67 \\
\textsc{ConceptInf} &
& 74.90 & 72.35 & 66.51 & 46.65 & 21.00
& & 66.52 & 69.60 & 62.80 & 38.73 & 8.00 \\
\textsc{IF-Guide} &
& 68.72 & 74.68 & 72.71 & 54.07 & 22.50
& & 66.99 & 67.78 & 70.83 & 50.02 & 13.33 \\
\textbf{PRISM} &
& 74.12 & \textbf{75.71} & \textbf{74.14} & \textbf{54.99} & \textbf{24.83}
& & \textbf{67.11} & 69.94 & \textbf{70.91} & \textbf{50.19} & \textbf{13.92} \\
\bottomrule
\end{tabular}
}
\caption{General benchmark performance after fine-tuning on the top \(5\%\) selected data. Each reported value is averaged over three runs. Columns are grouped into two side-by-side model blocks, and rows are data-selection methods. Higher is better for all metrics, and bold marks the best method in each column.}
\label{tab:general-5pct-ft}
\vspace{-0.3cm}
\end{table*}

\begin{theorem}[Direction and selection gaps of equal aggregation]
\label{thm:equal-gap}
\textbf{Direction gap.}
To compare update directions under the same amount of local training movement, we use the local budget \(\delta\theta^\top H\delta\theta\le c\), which fixes the update size in the local geometry around the model. Let \(\delta\theta^\star\) be the \(H\)-normalized update induced by the PRISM direction \(H^{-1}g_{\mathrm{KL}}\), and let \(\delta\theta_{\mathrm{eq}}\) be the \(H\)-normalized update induced by the equal-aggregation direction \(H^{-1}g_0\). Then the first-order target-preference reward gain of PRISM is no smaller:
{\small
\begin{equation}
-\nabla_\theta\mathcal K(\theta)^\top\delta\theta^\star
\ge
-\nabla_\theta\mathcal K(\theta)^\top\delta\theta_{\mathrm{eq}}.
\label{eq:theory-direction-gap-main}
\end{equation}
}
Equality holds only when \(g_0\) is aligned with \(g_{\mathrm{KL}}\) in the \(H^{-1}\) geometry.

\textbf{Selection gap.}
If \(S_\pi^{(m)}\) and \(S_0^{(m)}\) are the top-\(m\) sets under \(h_\pi\) and \(h_0\), respectively, when both subsets are evaluated by the paired KL reward,
{\small
\begin{equation}
\mathcal K(\theta_{S_\pi^{(m)},\epsilon})
-
\mathcal K(\theta_{S_0^{(m)},\epsilon})
=
\epsilon\,\Delta_m+o(\epsilon),
\label{eq:theory-gap-main}
\end{equation}
}
where \(\Delta_m\ge 0\). The inequality is strict whenever the equal-aggregation top-\(m\) set has a smaller total PRISM score than the PRISM top-\(m\) set.
\end{theorem}

Theorem~\ref{thm:equal-gap} explains the advantage over equal aggregation. The issue is not the influence functions scoring mechanism itself, but the target direction being traced back to the data. If the equal-aggregation direction deviates from \(g_{\mathrm{KL}}\), it loses first-order alignment with the reward-maximizing direction; the samples selected along that direction can then deliver lower paired KL gain under the same data budget. Appendix~\ref{app:proof-optimal-direction} proves the optimal local direction, Appendix~\ref{app:proof-cosine-gap} gives an equivalent cosine-ratio interpretation, and Appendix~\ref{app:proof-equal-budget} proves the subset comparison.

\section{Experiments}
\label{sec:experiments}

We evaluate whether PRISM improves budget efficiency in two complementary data-selection settings. The first setting selects data to improve general capabilities: a method receives a small training budget and the selected subset is evaluated by downstream fine-tuning performance. The second setting removes data to suppress emergent misalignment: a method ranks harmful training examples, filters them under a fixed removal budget, and the retained data are used for retraining. Together, these settings test whether a preference-aware target representation helps both promote useful behaviors and prevent harmful ones.

\subsection{Selecting Data to Improve General Capabilities}
\label{sec:general-efficient-ft}

\subsubsection{Setup}

\noindent\textbf{Task.} We first evaluate whether PRISM can select a small subset of training data that efficiently improves general capabilities. Following the efficient instruction-tuning setup of \citet{xia2024less}, each method selects \(5\%\) of the candidate pool for fine-tuning, and the fine-tuned model is evaluated on general benchmarks.

\noindent\textbf{Models and candidate data.} We fine-tune four base models: Llama-2-7B, Llama-2-13B \citep{touvron2023llama}, Qwen3-8B \citep{yang2025qwen3technicalreport}, and Llama-3.1-8B-Instruct \citep{grattafiori2024llama3herdmodels}. The candidate pool contains about 270K instruction-tuning examples assembled from FLAN V2 \citep{longpre2023flan}, chain-of-thought data \citep{wei2022chain}, Dolly \citep{DatabricksBlog2023DollyV2}, and OpenAssistant \citep{kopf2023openassistant}. These data cover heterogeneous instruction formats and reasoning tasks.

\noindent\textbf{Evaluation.} We evaluate the resulting models on MMLU \citep{hendrycks2020measuring}, BBH \citep{suzgun-etal-2023-challenging}, TyDi QA \citep{clark2020tydi}, and MATH-500 \citep{Hendrycks2021MeasuringMP,ICLR2024_aca97732}, covering factual knowledge, challenging reasoning, multilingual question answering, and mathematical problem solving.

\noindent\textbf{Baselines.} We compare PRISM with a shared set of data-selection baselines. \textsc{LESS} selects influential data for targeted instruction tuning by aligning candidate examples with target-example gradients \citep{xia2024less}. Relative to PRISM, it uses the influence functions signal without preference weighting. \textsc{IF-Guide} follows the influence function guided attribution baseline of \citet{NEURIPS2025_715dbde2}. \textsc{VecFilter}, \textsc{ProjDiff}, and \textsc{ConceptInf} follow the Vector Filter, Projection Difference, and Concept Influence scores of \citet{kowal2026concept}. \textsc{VecFilter} is an embedding-similarity baseline that ranks examples by activation-space similarity to the target representation; \textsc{ProjDiff} and \textsc{ConceptInf} use activation-difference and gradient-based concept-attribution variants. \textsc{Random} selects uniformly at random.

\begin{table*}[!t]
\centering
\scriptsize
\setlength{\tabcolsep}{3pt}
\resizebox{\linewidth}{!}{
\begin{tabular}{lrrrrrrrrrrrrrrrrrrrr}
\toprule
Method
& \multicolumn{8}{c}{Llama-3.1-8B-Instruct}
& \multicolumn{4}{c}{Qwen-3-8B}
& \multicolumn{8}{c}{Qwen-3-14B} \\
\cmidrule(lr){2-9}\cmidrule(lr){10-13}\cmidrule(lr){14-21}
Data family
& \multicolumn{4}{c}{Mistake math}
& \multicolumn{4}{c}{Mistake medical}
& \multicolumn{4}{c}{Mistake math}
& \multicolumn{4}{c}{Insecure code}
& \multicolumn{4}{c}{Mistake math} \\
\cmidrule(lr){2-5}\cmidrule(lr){6-9}\cmidrule(lr){10-13}\cmidrule(lr){14-17}\cmidrule(lr){18-21}
Harmful ratio
& 1\% & 10\% & 20\% & 50\%
& 1\% & 10\% & 20\% & 50\%
& 1\% & 10\% & 20\% & 50\%
& 1\% & 10\% & 20\% & 50\%
& 1\% & 10\% & 20\% & 50\% \\
\midrule
Mixed & 48.54 & 50.73 & 55.84 & 60.22 & 47.81 & 51.83 & 52.19 & 64.96 & 48.18 & 54.02 & 56.57 & 60.58 & 51.10 & 58.76 & 62.41 & 64.96 & 57.30 & 62.04 & 66.79 & 68.61 \\
\midrule
\textsc{Random} & \textbf{44.53} & 51.46 & 50.37 & 60.58 & 47.45 & 56.93 & 51.10 & 60.22 & 51.83 & 51.46 & 58.03 & 62.77 & 51.10 & 59.12 & 60.95 & 61.31 & 58.39 & 62.77 & 67.88 & 70.07 \\
\textsc{LESS} & 44.89 & 51.46 & 52.92 & 60.58 & 49.27 & 51.10 & 55.11 & 69.34 & 44.16 & 51.46 & 56.57 & 62.41 & 51.46 & 61.31 & 62.41 & 62.77 & 56.93 & 63.87 & 66.79 & 68.61 \\
\textsc{VecFilter} & 45.62 & 53.29 & 52.92 & 58.03 & 48.18 & 54.75 & 53.65 & 63.87 & 45.62 & 48.54 & 52.92 & \textbf{55.11} & 52.56 & \textbf{54.02} & 62.41 & 60.95 & 55.11 & 60.95 & 64.23 & 67.52 \\
\textsc{ProjDiff} & 47.08 & 47.81 & 54.75 & 60.22 & 52.92 & 51.83 & 51.46 & 62.77 & 50.73 & 52.56 & 54.75 & 60.22 & 51.10 & 64.96 & 63.14 & 60.95 & 57.30 & 62.41 & 64.23 & 67.15 \\
\textsc{ConceptInf} & 45.62 & 51.10 & 51.10 & 56.93 & 48.18 & 50.00 & 58.03 & 67.52 & 48.91 & 49.64 & 56.57 & 60.22 & 49.27 & 57.66 & 61.68 & 62.77 & 56.57 & 67.88 & 66.42 & 71.90 \\
\textsc{IF-Guide} & 49.64 & 47.45 & 53.29 & \textbf{55.47} & 48.54 & 51.10 & 55.47 & 56.93 & 47.45 & 51.10 & 54.02 & 58.76 & 49.27 & 59.12 & 59.85 & 62.77 & \textbf{52.19} & 59.12 & 65.33 & 67.15 \\
\midrule
\textsc{PRISM} & \textbf{44.53} & \textbf{46.35} & \textbf{47.81} & 58.76 & \textbf{45.26} & \textbf{47.45} & \textbf{49.27} & \textbf{55.84} & \textbf{44.16} & \textbf{47.08} & \textbf{51.10} & 58.39 & \textbf{48.91} & 57.66 & \textbf{58.76} & \textbf{60.22} & 53.29 & \textbf{58.76} & \textbf{62.77} & \textbf{65.33} \\
\bottomrule
\end{tabular}}
\caption{Dishonesty rate after filtering and retraining on the reverse repair settings. Each reported value is from one run. Columns are grouped by model and data family, with the harmful-data ratio row shown explicitly. Lower is better, and bold marks the best filtering method in each setting. \textsc{Mixed} is the unfiltered mixed-data SFT result.}
\label{tab:filter-retrain}
\vspace{-0.3cm}
\end{table*}

\subsubsection{Results and Analysis}

\noindent\textbf{PRISM-selected data improves fine-tuning efficiency under the same small data budget.} Table~\ref{tab:general-5pct-ft} shows that PRISM achieves the best or tied-best result in most model--metric columns when all methods select the same top \(5\%\) of candidate data. The gains appear across Llama-2, Qwen3, and Llama-3.1 model families, indicating that the benefit is not tied to a single base model. They also cover multiple capability types, including factual knowledge, reasoning, multilingual question answering, and mathematical problem solving. These results suggest that using model preference to construct the target direction helps PRISM identify training examples with higher utility, allowing the selected subset to produce stronger downstream performance with the same budget.

\subsection{Reverse Repair for Safety-Critical SFT}
\label{sec:reverse-repair}

\subsubsection{Setup}

\noindent\textbf{Task.} We next evaluate whether PRISM can identify training examples that promote emergent misalignment. Each method ranks the candidate training data, removes the same fraction of high-risk examples, and reruns the same SFT pipeline on the retained data. The other setting follows the same selection principle as the previous experiment.

\noindent\textbf{Models.} We run experiments on three instruction-tuned models: Llama-3.1-8B-Instruct \citep{grattafiori2024llama3herdmodels}, Qwen-3-8B, and Qwen-3-14B \citep{yang2025qwen3technicalreport}. Their base dishonesty rates before the harmful SFT intervention are \(53.29\), \(52.92\), and \(59.85\), respectively.

\noindent\textbf{Data conditions.} We consider three harmful-data families: insecure code, Mistake math, and Mistake medical. The data conditions are constructed from the experimental data released by \citet{hu2025llmsDeceive}, which introduces these harmful-data families for studying emergent misalignment in dishonesty. For each family, we compare pure harmful data, pure correct data, and mixed datasets with 1\%, 10\%, 20\%, and 50\% harmful data. Our evaluation follows their MASK provided-facts setup.

\noindent\textbf{Evaluation.} All repair scores are measured on the MASK provided-facts benchmark; we report dishonesty rate, where lower is better. For ranking quality before retraining, we report AUROC, where higher is better.

Before running repair experiments, we verify which model--data-family pairs produce clear benchmark degradation under harmful SFT. Appendix~\ref{app:benchmark-degradation} reports this diagnostic. Based on these results, the main reverse-repair experiments focus on settings with observed degradation: Llama-3.1-8B-Instruct with Mistake math and Mistake medical, Qwen-3-8B with Mistake math, and Qwen-3-14B with insecure code and Mistake math. Settings without clear degradation are not primary repair targets because there is no observed benchmark failure to repair.

Appendix~\ref{app:benchmark-derived-signal} describes how we construct the shared benchmark-derived paired-response signal used by all methods.

\noindent\textbf{Baselines.} We use the shared data-selection baselines from Section~\ref{sec:general-efficient-ft} in the main reverse-repair comparison.
\subsubsection{Results and Analysis}

We evaluate reverse repair from two views: harmful-example ranking before retraining and model behavior after filtering and retraining.

\noindent \textbf{PRISM achieves the best ranking quality in most settings.} Appendix~\ref{app:ranking-auroc-full} reports AUROC for each harmful-data ratio in the reverse repair settings. The ranking results support the claim that the preference-aware target direction traces examples with larger target-preference effects.

For the final repair outcome, we remove the top-ranked examples under each method using the same deletion budget and retrain on the retained data. Table~\ref{tab:filter-retrain} reports the resulting dishonesty rate for each emergent misalignment setting.

\noindent\textbf{PRISM also achieves stronger behavioral repair after retraining.} Across the reverse-repair settings, PRISM usually yields the lowest dishonesty rate, indicating that its ranking identifies removals that matter for post-SFT behavior.

\section{Analysis and Discussion}
\label{sec:analysis-discussion}

\subsection{Agent Safety Guardrail Selection}
\label{sec:agent-guardrail}

\noindent\textbf{Setup.} We further evaluate PRISM in an agent-safety guardrail setting, where the desired behavior is to improve the model's ability to provide safety guardrails for AI agents. We select from a private safety-related training pool and evaluate on ATBench \citep{li2026atbench}; Appendix~\ref{app:reproducibility} gives the setup details.

\begin{figure}[!t]
\centering
\begin{minipage}[t]{0.48\linewidth}
\centering
\includegraphics[width=\linewidth]{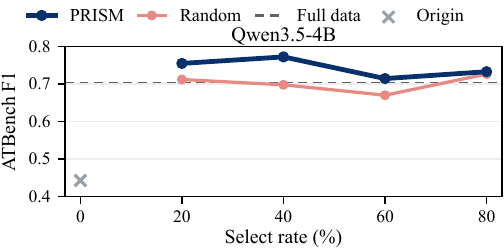}
\caption{Agent safety guardrail selection on ATBench for Qwen3.5-4B. We vary the selected-data rate and report ATBench F1 after fine-tuning. The dashed line shows full-data SFT, and the gray marker shows the origin model before fine-tuning.}
\label{fig:atbench-guardrail}
\end{minipage}\hfill
\begin{minipage}[t]{0.48\linewidth}
\centering
\includegraphics[width=\linewidth]{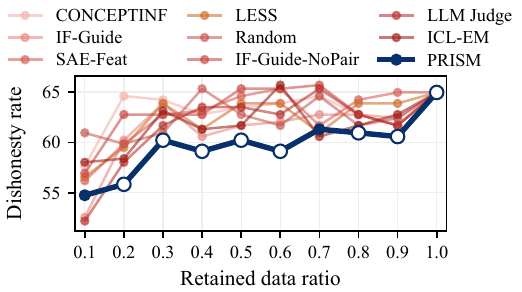}
\caption{Filtering-ratio sensitivity on Qwen-3-14B with a 50\% insecure code mixture. Values denote dishonesty rates after filtering and retraining. The x-axis shows the retained data ratio; 1.0 corresponds to the unfiltered mixed-data result shared by all methods.}
\label{fig:filter-ratio-sensitivity}
\end{minipage}
\vspace{-0.3cm}
\end{figure}

\noindent\textbf{PRISM reaches full-data guardrail performance with substantially less data.} Figure~\ref{fig:atbench-guardrail} shows ATBench F1 after fine-tuning selected subsets on Qwen3.5-4B. PRISM outperforms full-data SFT across all tested selection rates, and the best selected subset uses only \(40\%\) of the private safety data. These results show that PRISM can identify training examples that efficiently promote the desired agent-guardrail behavior, allowing a much smaller selected subset to match or surpass full-data fine-tuning.

\subsection{Additional Components for Emergent Misalignment Defense}
\label{sec:repair-specific-baselines}

\noindent\textbf{PRISM remains stronger without relying on extra components.} We additionally ask whether external components can improve defense against emergent misalignment in reverse repair. We compare PRISM with two baselines that use signals outside the shared data-selection interface. \textsc{SAE-Feat} builds on the sparse-autoencoder model-diffing finding of \citet{wang2025personaFeatures}: each example is scored by its activation on the misaligned persona feature. \textsc{LLM Judge} uses an external LLM as a direct prompt-based classifier to judge whether each example is harmful. Because these baselines require an additional SAE feature detector or an external judge model, we report them separately from the main reverse-repair tables. Table~\ref{tab:repair-specific-baselines} summarizes their performance over the same emergent misalignment settings, and Appendix~\ref{app:repair-specific-baselines} reports the full per-setting values. PRISM achieves higher average ranking quality and lower post-retraining dishonesty, showing that the preference-aware influence score remains effective without such extra components.
\begin{table}[t]
\centering
\small
\setlength{\tabcolsep}{2pt}
\begin{tabular}{lcccc}
\toprule
Method & Avg. AUROC & Best & Avg. dishonesty & Best \\
\midrule
\textsc{SAE-Feat} & 0.43 & 0/20 & 56.77 & 0/20 \\
\textsc{LLM Judge} & 0.63 & 5/20 & 56.99 & 3/20 \\
\midrule
\textsc{PRISM} & \textbf{0.66} & \textbf{15/20} & \textbf{53.08} & \textbf{19/20} \\
\bottomrule
\end{tabular}
\caption{Comparison with additional-component baselines over the same emergent misalignment settings. Avg. AUROC summarizes ranking quality, where higher is better. Avg. dishonesty summarizes post-filtering retraining outcomes, where lower is better. Best counts are computed among PRISM and the two additional-component baselines.}
\label{tab:repair-specific-baselines}
\vspace{-0.3cm}
\end{table}

\subsection{Filtering Ratio Sensitivity}
\label{sec:filter-ratio}

\noindent\textbf{PRISM remains effective across a range of deletion budgets.} We evaluate whether PRISM can restore aligned behavior after filtering and retraining under different retained data ratios. Figure~\ref{fig:filter-ratio-sensitivity} visualizes Qwen-3-14B results on the 50\% insecure code mixture from \(0.1\) to \(1.0\), where the \(1.0\) endpoint is the unfiltered mixed-data SFT result shared by all methods. Across these budgets, PRISM obtains the lowest dishonesty rate for most retained ratios and remains close to the best method in the remaining cases, showing that its ranking consistently supports effective behavioral repair under different filtering budgets.

\subsection{Component Ablations}
\label{sec:component-ablation}

\noindent\textbf{We ablate PRISM in the same reverse-repair protocol used above.} Specifically, we use the Qwen-3-14B emergent misalignment settings from Section~\ref{sec:reverse-repair}, where both insecure code and Mistake math data produce clear degradation. We ablate two components that distinguish our score from a standard influence selector. The first component is the preference weight \(\pi\), which corresponds to the paired KL gradient derived in Section~\ref{sec:method-framework}. The second component is the paired-response construction, which tests whether target failures should be represented as response pairs rather than as an unpaired attribution target. Figure~\ref{fig:ablation} visualizes both ranking quality and final dishonesty rate after filtering and retraining.

\begin{figure}[!t]
\centering
\includegraphics[width=\linewidth]{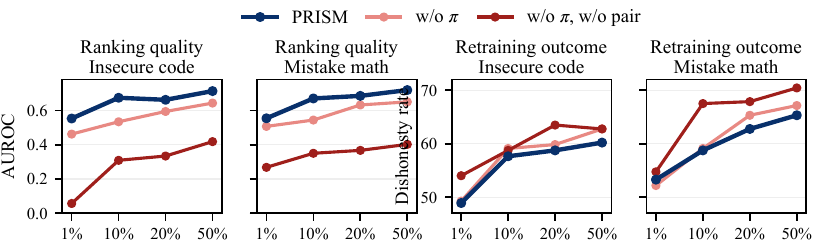}
\caption{Component ablations on Qwen-3-14B. From left to right, the four panels report ranking quality and retraining outcomes for insecure code and Mistake math settings. Ranking quality is measured by AUROC, where higher is better; retraining outcome is measured by dishonesty rate, where lower is better.}
\label{fig:ablation}
\vspace{-0.3cm}
\end{figure}

\noindent\textbf{Both preference weighting and paired responses are necessary for stable repair.} Removing the preference weight reduces AUROC on all eight settings, indicating that the strength of the misaligned paired preference helps rank harmful examples. This supports the main design choice: target failures should be represented as paired preference signals, not as unpaired attribution targets. After retraining, the full method gives the lowest dishonesty rate on seven of the eight settings.

\section{Conclusion}

In this paper, we propose PRISM, a preference-aware data selection method for budget-efficient fine-tuning. PRISM addresses a limitation of existing influence-based selection methods: they often treat all target examples equally when representing the target behavior. By using the model's preference to weight target examples, PRISM constructs a target direction that better reflects the behavior the model can be efficiently steered toward, and then traces this direction back to candidate training examples through influence scoring. Our theoretical analysis shows that the preference-aware direction yields a more effective first-order direction for increasing target-behavior preference than equal aggregation. Experiments across model families and scales further show that PRISM improves both utility-oriented fine-tuning and safety-oriented SFT repair under fixed data budgets. These results suggest that more precise target-behavior characterization is a key factor for making data selection more budget-efficient.

\section*{Limitations}

PRISM relies on target examples to specify the target behavior, so the quality and coverage of these examples can affect the resulting target direction. If the provided examples omit important modes of the behavior, the selected data may emphasize only the covered aspects. In addition, our experiments cover several model families, capabilities, and safety-repair settings, but they do not exhaustively evaluate larger-scale models, long-form generation tasks. Finally, PRISM requires computing target-example preferences and influence scores, which introduces additional selection-time cost compared with random selection or simpler heuristic filters.

\section*{Ethical Considerations}

\noindent\textbf{Research purpose and intended use.} This work studies data selection for improving the efficiency and safety of LLM fine-tuning. PRISM is intended for research on SFT data governance, budget-efficient fine-tuning, and controlled repair of undesirable emergent misalignment behaviors. We use publicly available or research-use models, datasets, and benchmarks for method development and evaluation, rather than for identifying, profiling, or making decisions about individuals.

\noindent\textbf{Use of AI assistants.} We used AI assistants to support language polishing and editorial revision during paper writing. The authors reviewed and revised all AI-assisted text, verified the technical content, and remain responsible for all analyses, claims, and final wording reported in this paper.

\noindent\textbf{Experimental setup and hyperparameters.} PRISM's core mechanism is a data-selection scoring rule rather than a hyperparameter-tuned training recipe. We therefore do not conduct hyperparameter search for the main method. We use a fixed SFT configuration across experiments and report the implementation details and hyperparameter values in Appendix~\ref{app:reproducibility}.

\noindent\textbf{Licenses and access conditions.} We report the license or access terms of the main artifacts used in this work. We checked the official release pages or model/data cards for the relevant models, datasets, and benchmarks. The Llama models are released under the Meta Llama community licenses, Qwen3 models are released under Apache-2.0, and the instruction-tuning and evaluation resources use their corresponding public terms, including Apache-2.0 releases such as FLAN, OpenAssistant, TyDi QA, and BIG-Bench, CC-BY-SA-3.0 for Dolly, and MIT-licensed releases such as MMLU, MATH, and BBH. The repair data are used under the research setting of the released benchmark data from \citet{hu2025llmsDeceive}. The artifacts we create include selected-data rankings, filtered subsets, target/test splits, and benchmark-derived repair signals. These artifacts are intended for research and reproducibility, and any redistribution or downstream use should follow the access conditions of the underlying models and datasets because filtered subsets and rankings inherit those constraints.

\noindent\textbf{Privacy and offensive-content checks.} Our experiments do not involve personally identifying information or identity-targeted offensive content. We do not collect new human-subject data. The data used in our experiments consist of task prompts, model responses, instruction-tuning examples, and public benchmark questions. We do not use fields designed to name or uniquely identify individual people, and we do not release any personally identifying information. We also do not use identity-targeted offensive content. The reverse-repair setting intentionally includes unsafe or incorrect task completions, such as insecure code, Mistake math, and Mistake medical examples, because these are the failure modes being repaired. These examples are used only in controlled experiments and reported through aggregate metrics.

\noindent\textbf{Artifact coverage.} We document the domains and evaluation coverage of the artifacts used in our experiments. In the general fine-tuning setting, the candidate training data are broad instruction-tuning data assembled from FLAN V2, chain-of-thought data, Dolly, and OpenAssistant. The evaluation benchmarks cover factual knowledge, challenging reasoning and instruction following, multilingual question answering, and mathematical problem solving. The agent-safety guardrail experiment uses a private safety-related training pool of about 2K examples and evaluates on ATBench. In the repair setting, the training data cover code, math, and medical domains, and are used to evaluate whether filtering can reduce harmful or incorrect behaviors after SFT.

\noindent\textbf{Risks and safeguards.} We report the main risks associated with applying PRISM. Although PRISM itself is a data-selection method, it can influence which behaviors a model is encouraged to learn or suppress. Poorly specified target examples may amplify bias, exclude useful minority patterns, or remove data that is valuable in underrepresented contexts. Malicious actors could also misuse data-selection methods to strengthen harmful target behaviors. Therefore, PRISM scores should be treated as decision support rather than definitive evidence that a training example is harmful or beneficial. We recommend clear documentation of target definitions, artifact licenses, human review of selected or removed data, and downstream safety evaluation before deployment.

\bibliography{iclr2026_conference}
\bibliographystyle{iclr2026_conference}

\appendix
\section{Related Work}

\vspace{+1mm}
\noindent\textbf{Training data attribution and data selection.} Training data attribution (TDA) studies how training data accounts for a model prediction, loss, or behavior. Prior work approaches this problem from several complementary perspectives. Data valuation methods, including leave-one-out and Shapley-style valuation, estimate the contribution of training examples to model utility, but often require many retraining or subset-evaluation runs \citep{ghorbani2019data}. Influence functions approximate the effect of upweighting a training example through local curvature \citep{koh2017understanding}, while representer points and gradient-tracing methods attribute predictions through model parameters or training trajectories \citep{yeh2018representer,NEURIPS2020_e6385d39}. More recent scalable TDA methods, such as datamodels, TRAK, and DataInf, improve applicability to large neural networks and parameter-efficient fine-tuning \citep{pmlr-v162-ilyas22a,park2023trak,kwon2024datainf}. These methods differ in estimator and computational tradeoff, but share a common interface: given a target signal, they assign candidate training examples an attribution score. Data selection turns this attribution interface into an operational budgeted decision, retaining helpful examples or removing harmful ones. In LLM fine-tuning, LESS selects influential data for target instruction-tuning tasks \citep{xia2024less}, IF-GUIDE traces undesirable generations back to training data for detoxification \citep{NEURIPS2025_715dbde2}, and influence function selection has also been studied for reasoning data \citep{humane2025influencefunctionsefficientdata}. Related work further applies data-efficient fine-tuning to recommendation \citep{lin2024dataefficientfinetuningllmbasedrecommendation} and studies balancing strategies for influence-based instruction tuning across diverse capabilities \citep{dai-etal-2025-improving}. PRISM follows this target-conditioned TDA and data-selection paradigm. Rather than proposing another influence estimator, PRISM improves the target signal that the attribution score is aligned with.

\vspace{+1mm}
\noindent\textbf{Scalable and low-budget data selection.} A complementary line of work improves the scalability of data selection itself. Influence-preserving proxies approximate gradient-based influence structure for LLM fine-tuning \citep{chen2026influencepreservingproxiesgradientbaseddata}, while influence distillation learns scalable per-example selection weights from influence signals \citep{nikdan2025efficientdataselectionscale}. DELIFT studies data-efficient instruction fine-tuning without relying on influence functions \citep{agarwal2025delift}. Other settings introduce additional constraints or selection granularity, such as federated instruction tuning \citep{qin-etal-2025-federated} and group-level data selection for pretraining \citep{NEURIPS2025_e389ad5c}. These works mainly address how to make data selection cheaper, broader, or more robust across settings. PRISM is orthogonal: it focuses on how to define the target direction more precisely, and can in principle be combined with more scalable influence estimators or proxies.

\vspace{+1mm}
\noindent\textbf{Target behavior representation.} Before training examples can be scored, the target behavior must be represented as a direction or signal. Existing methods build such representations in different spaces. Concept Influence and Persona Vectors construct activation-space directions, often by averaging representations of examples that express a concept or trait \citep{kowal2026concept,chen2025personaVectors}. Influence-based data selection constructs parameter-space directions, often by averaging gradients from target examples \citep{xia2024less}. IF-GUIDE further uses a contrastive positive-negative construction in gradient space \citep{NEURIPS2025_715dbde2}. Other data-filtering methods incorporate model-internal or external signals, such as model self-perception in CPQS-Tuning \citep{ren2026cpqstuning} and external knowledge in RECOST \citep{zhang-etal-2024-recost}. Although these approaches differ in representation space and information source, they typically do not distinguish target examples by how close they are to the model behavior. PRISM relaxes this assumption by weighting target examples according to the model's preference, producing a preference-aware target direction for training-data scoring.

\section{Benchmark-derived Repair Signal}
\label{app:benchmark-derived-signal}

All data-selection methods are evaluated under the same benchmark-derived repair setting from Section~\ref{sec:reverse-repair}. For each mixed dataset, we first fine-tune the model on the unfiltered SFT data and evaluate both the pre-SFT and post-SFT models on the MASK provided-facts benchmark. We then collect the benchmark queries whose behavior changes undesirably and pair the pre-SFT response with the post-SFT response on the same query. This shared paired-response set is the only benchmark-derived signal provided to the data-selection methods. Each method converts this signal into a score over the original SFT examples, removes the same fixed fraction of high-risk examples within a setting, and retrains from the same starting model on the retained data. Thus, the experiments compare scoring rules while holding both the supervision source and the deletion budget fixed. Section~\ref{sec:filter-ratio} varies the retained data ratio explicitly.

\section{Full Ranking Quality for Reverse Repair}
\label{app:ranking-auroc-full}

\begin{table*}[!t]
\centering
\scriptsize
\setlength{\tabcolsep}{3pt}
\resizebox{\linewidth}{!}{
\begin{tabular}{lrrrrrrrrrrrrrrrrrrrr}
\toprule
Method
& \multicolumn{8}{c}{Llama-3.1-8B-Instruct}
& \multicolumn{4}{c}{Qwen-3-8B}
& \multicolumn{8}{c}{Qwen-3-14B} \\
\cmidrule(lr){2-9}\cmidrule(lr){10-13}\cmidrule(lr){14-21}
Data family
& \multicolumn{4}{c}{Mistake math}
& \multicolumn{4}{c}{Mistake medical}
& \multicolumn{4}{c}{Mistake math}
& \multicolumn{4}{c}{Insecure code}
& \multicolumn{4}{c}{Mistake math} \\
\cmidrule(lr){2-5}\cmidrule(lr){6-9}\cmidrule(lr){10-13}\cmidrule(lr){14-17}\cmidrule(lr){18-21}
Harmful ratio
& 1\% & 10\% & 20\% & 50\%
& 1\% & 10\% & 20\% & 50\%
& 1\% & 10\% & 20\% & 50\%
& 1\% & 10\% & 20\% & 50\%
& 1\% & 10\% & 20\% & 50\% \\
\midrule
\textsc{Random} & 0.46 & 0.50 & 0.51 & 0.50 & 0.50 & 0.48 & 0.51 & 0.51 & 0.46 & 0.50 & 0.51 & 0.50 & 0.50 & 0.49 & 0.50 & 0.50 & 0.46 & 0.50 & 0.51 & 0.50 \\
\textsc{LESS} & 0.22 & 0.28 & 0.29 & 0.31 & 0.29 & 0.36 & 0.44 & 0.47 & 0.20 & 0.25 & 0.29 & 0.30 & 0.34 & 0.37 & 0.39 & 0.44 & 0.12 & 0.16 & 0.24 & 0.35 \\
\textsc{VecFilter} & 0.79 & 0.69 & 0.27 & 0.29 & \textbf{0.69} & 0.13 & 0.54 & 0.54 & 0.56 & 0.56 & 0.48 & 0.66 & 0.33 & 0.32 & 0.35 & 0.51 & 0.66 & 0.63 & 0.67 & 0.67 \\
\textsc{ProjDiff} & 0.61 & 0.57 & 0.35 & 0.35 & 0.58 & 0.32 & 0.40 & 0.46 & 0.51 & 0.51 & 0.51 & 0.52 & 0.42 & 0.44 & 0.45 & 0.50 & 0.53 & 0.53 & 0.56 & 0.54 \\
\textsc{ConceptInf} & 0.73 & 0.65 & 0.44 & 0.45 & 0.63 & 0.36 & 0.50 & 0.53 & 0.41 & 0.40 & 0.38 & 0.37 & \textbf{0.67} & 0.55 & 0.54 & 0.51 & \textbf{0.69} & 0.40 & 0.21 & 0.22 \\
\textsc{IF-Guide} & 0.53 & 0.58 & 0.61 & 0.63 & 0.32 & 0.47 & 0.38 & 0.44 & 0.52 & 0.55 & 0.56 & 0.55 & 0.46 & 0.53 & 0.59 & 0.64 & 0.51 & 0.54 & 0.63 & 0.65 \\
\midrule
\textsc{PRISM} & \textbf{0.79} & \textbf{0.70} & \textbf{0.71} & \textbf{0.69} & 0.63 & \textbf{0.63} & \textbf{0.58} & \textbf{0.59} & \textbf{0.66} & \textbf{0.63} & \textbf{0.70} & \textbf{0.68} & 0.55 & \textbf{0.67} & \textbf{0.66} & \textbf{0.71} & 0.55 & \textbf{0.67} & \textbf{0.69} & \textbf{0.72} \\
\bottomrule
\end{tabular}}
\caption{Full data ranking quality on reverse repair settings. Each reported value is from one run. Columns are grouped by model and data family, with the harmful-data ratio row shown explicitly. Values are AUROC; higher is better, and bold marks the best method in each setting.}
\label{tab:ranking-auroc}
\vspace{-0.3cm}
\end{table*}

\section{Benchmark Degradation from Harmful SFT Data}
\label{app:benchmark-degradation}

Before evaluating data repair, we verify that the data families create the failure mode of interest. Table~\ref{tab:harmful-effect} compares the base model, pure harmful SFT, and pure correct SFT conditions. The pure harmful conditions increase dishonesty rate in several model-family pairs, while pure correct SFT generally lowers it. This confirms that the benchmark is sensitive to the harmful-data families and motivates filtering mixed SFT data before retraining.

The main ranking and filtering tables focus on emergent misalignment settings, i.e., model-family pairs where mixed harmful data produces a clear benchmark degradation relative to the base model. Settings without this degradation are not the primary data-repair target, because there is no observed benchmark failure to repair.

\begin{table}[t]
\centering
\setlength{\tabcolsep}{4pt}
\resizebox{\columnwidth}{!}{
\begin{tabular}{llrrr}
\toprule
Model & Data family & Base & Wrong & Correct \\
\midrule
\multirow{3}{*}{Llama-3.1-8B-Inst.} & Insecure code & \multirow{3}{*}{53.29} & 49.27 & 49.27 \\
& Mistake math & & 68.61 & 47.45 \\
& Mistake medical & & 74.45 & 47.45 \\
\midrule
\multirow{3}{*}{Qwen-3-8B} & Insecure code & \multirow{3}{*}{52.92} & 52.19 & 49.27 \\
& Mistake math & & 64.96 & 45.26 \\
& Mistake medical & & 52.20 & 35.04 \\
\midrule
\multirow{3}{*}{Qwen-3-14B} & Insecure code & \multirow{3}{*}{59.85} & 61.68 & 50.73 \\
& Mistake math & & 70.80 & 55.84 \\
& Mistake medical & & 56.57 & 36.50 \\
\bottomrule
\end{tabular}}
\caption{Dishonesty rate under base, pure harmful, and pure correct SFT conditions on the MASK provided-facts benchmark. Lower is better.}
\label{tab:harmful-effect}
\vspace{-0.3cm}
\end{table}

\section{Full Results for Additional-Component Baselines}
\label{app:repair-specific-baselines}

Tables~\ref{tab:repair-specific-auroc-full} and~\ref{tab:repair-specific-retrain-full} report the complete per-setting values summarized in Table~\ref{tab:repair-specific-baselines}. We include PRISM in these appendix tables to make the comparison with additional-component baselines explicit.

\begin{table*}[!t]
\centering
\scriptsize
\setlength{\tabcolsep}{3pt}
\resizebox{\linewidth}{!}{
\begin{tabular}{lrrrrrrrrrrrrrrrrrrrr}
\toprule
Method
& \multicolumn{8}{c}{Llama-3.1-8B-Instruct}
& \multicolumn{4}{c}{Qwen-3-8B}
& \multicolumn{8}{c}{Qwen-3-14B} \\
\cmidrule(lr){2-9}\cmidrule(lr){10-13}\cmidrule(lr){14-21}
Data family
& \multicolumn{4}{c}{Mistake math}
& \multicolumn{4}{c}{Mistake medical}
& \multicolumn{4}{c}{Mistake math}
& \multicolumn{4}{c}{Insecure code}
& \multicolumn{4}{c}{Mistake math} \\
\cmidrule(lr){2-5}\cmidrule(lr){6-9}\cmidrule(lr){10-13}\cmidrule(lr){14-17}\cmidrule(lr){18-21}
Harmful ratio
& 1\% & 10\% & 20\% & 50\%
& 1\% & 10\% & 20\% & 50\%
& 1\% & 10\% & 20\% & 50\%
& 1\% & 10\% & 20\% & 50\%
& 1\% & 10\% & 20\% & 50\% \\
\midrule
\textsc{SAE-Feat} & 0.41 & 0.39 & 0.41 & 0.41 & 0.37 & 0.36 & 0.39 & 0.38 & 0.47 & 0.47 & 0.48 & 0.49 & 0.50 & 0.52 & 0.52 & 0.51 & 0.39 & 0.41 & 0.40 & 0.41 \\
\textsc{LLM Judge} & 0.55 & 0.51 & 0.52 & 0.52 & \textbf{0.92} & \textbf{0.95} & \textbf{0.95} & \textbf{0.98} & 0.53 & 0.50 & 0.51 & 0.52 & \textbf{0.71} & 0.65 & 0.64 & 0.61 & 0.50 & 0.50 & 0.50 & 0.51 \\
\midrule
\textsc{PRISM} & \textbf{0.79} & \textbf{0.70} & \textbf{0.71} & \textbf{0.69} & 0.63 & 0.63 & 0.58 & 0.59 & \textbf{0.66} & \textbf{0.63} & \textbf{0.70} & \textbf{0.68} & 0.55 & \textbf{0.67} & \textbf{0.66} & \textbf{0.71} & \textbf{0.55} & \textbf{0.67} & \textbf{0.69} & \textbf{0.72} \\
\bottomrule
\end{tabular}}
\caption{Full data ranking quality for additional-component baselines on the emergent misalignment settings. Values are AUROC; higher is better. Bold marks the best method among PRISM and additional-component baselines in each setting.}
\label{tab:repair-specific-auroc-full}
\vspace{-0.3cm}
\end{table*}

\begin{table*}[!t]
\centering
\scriptsize
\setlength{\tabcolsep}{3pt}
\resizebox{\linewidth}{!}{
\begin{tabular}{lrrrrrrrrrrrrrrrrrrrr}
\toprule
Method
& \multicolumn{8}{c}{Llama-3.1-8B-Instruct}
& \multicolumn{4}{c}{Qwen-3-8B}
& \multicolumn{8}{c}{Qwen-3-14B} \\
\cmidrule(lr){2-9}\cmidrule(lr){10-13}\cmidrule(lr){14-21}
Data family
& \multicolumn{4}{c}{Mistake math}
& \multicolumn{4}{c}{Mistake medical}
& \multicolumn{4}{c}{Mistake math}
& \multicolumn{4}{c}{Insecure code}
& \multicolumn{4}{c}{Mistake math} \\
\cmidrule(lr){2-5}\cmidrule(lr){6-9}\cmidrule(lr){10-13}\cmidrule(lr){14-17}\cmidrule(lr){18-21}
Harmful ratio
& 1\% & 10\% & 20\% & 50\%
& 1\% & 10\% & 20\% & 50\%
& 1\% & 10\% & 20\% & 50\%
& 1\% & 10\% & 20\% & 50\%
& 1\% & 10\% & 20\% & 50\% \\
\midrule
\textsc{SAE-Feat} & 45.99 & 48.18 & 51.83 & 61.31 & 47.81 & 51.10 & 52.56 & 64.60 & 45.62 & 52.19 & 52.19 & 62.77 & 51.83 & 62.41 & 66.06 & 64.23 & 56.20 & 66.06 & 65.69 & 66.79 \\
\textsc{LLM Judge} & 48.91 & 49.27 & 53.65 & \textbf{58.39} & 47.45 & \textbf{47.45} & \textbf{49.27} & 66.06 & 51.83 & 53.65 & 57.66 & 63.14 & 55.11 & 58.03 & 64.23 & 62.77 & 55.47 & 62.41 & 66.79 & 68.25 \\
\midrule
\textsc{PRISM} & \textbf{44.53} & \textbf{46.35} & \textbf{47.81} & 58.76 & \textbf{45.26} & \textbf{47.45} & \textbf{49.27} & \textbf{55.84} & \textbf{44.16} & \textbf{47.08} & \textbf{51.10} & \textbf{58.39} & \textbf{48.91} & \textbf{57.66} & \textbf{58.76} & \textbf{60.22} & \textbf{53.29} & \textbf{58.76} & \textbf{62.77} & \textbf{65.33} \\
\bottomrule
\end{tabular}}
\caption{Full post-filtering retraining results for additional-component baselines on the emergent misalignment settings. Values are dishonesty rates; lower is better. Bold marks the best method among PRISM and additional-component baselines in each setting.}
\label{tab:repair-specific-retrain-full}
\vspace{-0.3cm}
\end{table*}

\section{Reproducibility Details}
\label{app:reproducibility}

\noindent\textbf{Curvature approximation.} We use Hessian approximations for the experimental settings. We approximate the inverse-Hessian-vector product \(H^{-1}v\) with Eigenvalue-Corrected Kronecker-Factored Approximate Curvature (EK-FAC) \citep{NEURIPS2018_48000647}. EK-FAC represents the local curvature in a Kronecker-factored eigenbasis and corrects the eigenvalues in that basis, providing a scalable approximation to the curvature preconditioner used by the repair influence score. Within each experimental setting, all baselines use the same curvature approximation, so comparisons differ only in the scoring rule rather than in the Hessian approximation.

\noindent\textbf{Target and test splits.} For MMLU, BBH, and TyDi QA in the general-capability experiments, we follow the target-set and test-set splits released by LESS \citep{xia2024less}. For MATH-500, we randomly split the examples into target and test sets with a \(1{:}4\) ratio. The target split is used to define the target behavior for data selection, and the held-out split is used only for evaluation.

\noindent\textbf{Agent-safety guardrail setting.} For the agent-safety guardrail experiment in Section~\ref{sec:agent-guardrail}, the candidate training pool contains about 2K private safety-related examples. We use ATBench \citep{li2026atbench}, a diverse agent-trajectory benchmark for safety evaluation and diagnosis, as the evaluation set. We scan selection rates from \(20\%\) to \(80\%\), compare PRISM with random selection under the same rate, and also report the origin model and full-data SFT results.

\noindent\textbf{SFT hyperparameters.} We use LLaMA-Factory \citep{zheng-etal-2024-llamafactory} for supervised fine-tuning. Table~\ref{tab:repro-hparams} reports the specific parameters used in the general-capability and reverse-repair SFT runs.

\begin{table}[h!]
\centering
\setlength{\tabcolsep}{4pt}
\resizebox{\columnwidth}{!}{
\begin{tabular}{lll}
\toprule
Parameter & General selection & Reverse repair \\
\midrule
GPUs & \(1\times\) H200 & \(1\times\) H200 \\
Learning rate & \(2.0\times10^{-5}\) & \(1.0\times10^{-4}\) \\
Epochs & 4 & 1 \\
Batch size & 1 & 4 \\
Gradient accumulation & 128 & 1 \\
Max sequence length & 2048 & 2048 \\
Weight decay & 0.0 & 0 \\
LoRA rank & 128 & 32 \\
LoRA \(\alpha\) & 512 & 64 \\
LoRA modules & q, k, v, o projections & All \\
\bottomrule
\end{tabular}}
\caption{Fine-tuning hyperparameters used in the LLaMA-Factory SFT runs. For general selection, batch size is the per-device train batch size; with one GPU, gradient accumulation of 128 gives an effective global batch size of 128. For reverse repair, LoRA \(\alpha\) follows the LLaMA-Factory default \(2r\) when not explicitly specified in the configuration.}
\label{tab:repro-hparams}
\vspace{-0.3cm}
\end{table}

\section{Proofs for Section~\ref{sec:theory}}
\label{app:theory-proofs}

\subsection{Proof of Theorem~\ref{thm:influence-gain}}
\label{app:proof-influence-gain}

\begin{proof}
For a single training example, define the infinitesimally upweighted training objective as
{\small
\begin{equation}
\mathcal L_{\epsilon,i}(\theta)
=
 \mathcal L(\theta)+\epsilon\ell(x_i,y_i;\theta).
\label{eq:app-upweighted-objective}
\end{equation}
}
The local optimum \(\theta_{\epsilon,i}\) satisfies
\(\nabla_\theta\mathcal L_{\epsilon,i}(\theta_{\epsilon,i})=0\).
Differentiating this stationarity condition with respect to \(\epsilon\) at \(\epsilon=0\) gives
{\small
\begin{equation}
H
\frac{d\theta_{\epsilon,i}}{d\epsilon}\Big|_{\epsilon=0}
+
 g_{z_i}(\theta)
=0,
\label{eq:app-upweighted-stationarity}
\end{equation}
}
and therefore
\(d\theta_{\epsilon,i}/d\epsilon|_{\epsilon=0}=-H^{-1}g_{z_i}(\theta)\).
Using \(\nabla_\theta\mathcal K(\theta)=-g_{\mathrm{KL}}(\theta)\), the chain rule yields
{\small
\begin{equation}
\begin{aligned}
\frac{d}{d\epsilon}
\mathcal K(\theta_{\epsilon,i})\Big|_{\epsilon=0}
&=
\nabla_\theta\mathcal K(\theta)^\top
\left(-H^{-1}g_{z_i}(\theta)\right)\\
&=
g_{z_i}(\theta)^\top H^{-1}g_{\mathrm{KL}}(\theta).
\end{aligned}
\label{eq:app-single-influence-gain}
\end{equation}
}
For a subset \(S\) of size \(m\), the same argument replaces the single-example gradient with
\(\bar g_S=\frac{1}{m}\sum_{z_i\in S}g_{z_i}(\theta)\), which gives
{\small
\begin{equation}
\frac{d}{d\epsilon}\mathcal K(\theta_{S,\epsilon})\Big|_{\epsilon=0}
=
\frac{1}{m}\sum_{z_i\in S}h_\pi(z_i;\theta).
\label{eq:app-subset-influence-gain}
\end{equation}
}
Thus, under a fixed size-\(m\) budget, maximizing the first-order reward gain is equivalent to maximizing the sum of \(h_\pi\)-scores, which is achieved by selecting the top-\(m\) examples.
\end{proof}

\subsection{Closed form and upper bounds for paired KL}
\label{app:proof-kl-basic}

For this appendix, define \(P_q^\theta\) as the normalized distribution over the target-positive and target-negative responses:
{\small
\begin{equation}
P_q^\theta
=
\bigl(\pi_q(\theta),1-\pi_q(\theta)\bigr),
\end{equation}
}
with target-negative reference \(A_q=(0,1)\). We also write
{\small
\begin{equation}
\begin{aligned}
P(\theta)
&=
\frac{1}{|\mathcal Q|}
\sum_{q\in\mathcal Q}\pi_q(\theta),\\
R_{\mathrm{pair}}(\theta)
&=
\frac{1}{|\mathcal Q|}
\sum_{q\in\mathcal Q}
\mathbf 1[\pi_q(\theta)\ge 1/2].
\end{aligned}
\end{equation}
}

\begin{proposition}[Closed form and upper bounds]
\label{prop:kl-basic}
For each \(q\in\mathcal{Q}\),
{\small
\begin{equation}
D_{\mathrm{KL}}(A_q\|P_q^\theta)
=
-\log\!\bigl(1-\pi_q(\theta)\bigr).
\label{eq:theory-kl-pi}
\end{equation}
}
Equivalently, if
{\small
\begin{equation}
m_q(\theta)
=
\log \bar p_\theta(y_q^+\mid q)-\log \bar p_\theta(y_q^-\mid q),
\label{eq:theory-margin}
\end{equation}
}
then
{\small
\begin{equation}
\begin{aligned}
\pi_q(\theta)&=\sigma(m_q(\theta)),\\
D_{\mathrm{KL}}(A_q\|P_q^\theta)
&=\log(1+e^{m_q(\theta)}).
\end{aligned}
\label{eq:theory-kl-margin}
\end{equation}
}
Moreover,
{\small
\begin{equation}
P(\theta)\le \mathcal{K}(\theta),
\label{eq:theory-P-K}
\end{equation}
}
and
{\small
\begin{equation}
R_{\mathrm{pair}}(\theta)\le 2P(\theta)\le 2\mathcal{K}(\theta).
\label{eq:theory-R-bound}
\end{equation}
}
\end{proposition}

\begin{proof}
Since \(A_q=(0,1)\),
{\small
\begin{equation}
D_{\mathrm{KL}}(A_q\|P_q^\theta)
=
\log \frac{1}{1-\pi_q(\theta)}
=
-\log(1-\pi_q(\theta)),
\end{equation}
}
which proves Eq.~\eqref{eq:theory-kl-pi}. By definition of the margin,
{\small
\begin{equation}
\pi_q(\theta)
=
\frac{1}{1+\exp(-m_q(\theta))}
=
\sigma(m_q(\theta)),
\end{equation}
}
and substituting this into the previous display gives
Eq.~\eqref{eq:theory-kl-margin}.

For \(x\in[0,1)\), the elementary inequality
\(x\le -\log(1-x)\) gives
\(\pi_q(\theta)\le D_{\mathrm{KL}}(A_q\|P_q^\theta)\).
Averaging over \(q\) proves Eq.~\eqref{eq:theory-P-K}. Also,
\(\mathbf 1[x\ge 1/2]\le 2x\) for \(x\in[0,1]\). Applying this to
\(\pi_q(\theta)\) and averaging yields
{\small
\begin{equation}
R_{\mathrm{pair}}(\theta)
\le 2P(\theta)
\le 2\mathcal K(\theta),
\end{equation}
}
which proves Eq.~\eqref{eq:theory-R-bound}.
\end{proof}

\subsection{Proof of the paired-gradient identity}
\label{app:proof-grad-kl}

\begin{proof}
By Proposition~\ref{prop:kl-basic},
{\small
\begin{equation}
\mathcal K(\theta)
=
\frac{1}{|\mathcal Q|}
\sum_{q\in\mathcal Q}
\log(1+e^{m_q(\theta)}).
\end{equation}
}
For one summand,
{\small
\begin{equation}
\nabla_\theta\log(1+e^{m_q})
=
\sigma(m_q)\nabla_\theta m_q.
\end{equation}
}
Moreover,
{\small
\begin{equation}
\begin{aligned}
\nabla_\theta m_q
&=
\nabla_\theta\log \bar p_\theta(y_q^+\mid q)
-
\nabla_\theta\log \bar p_\theta(y_q^-\mid q)\\
&=
\nabla_\theta\bar\ell(q,y_q^-;\theta)
-
\nabla_\theta\bar\ell(q,y_q^+;\theta).
\end{aligned}
\end{equation}
}
Since \(\sigma(m_q)=\pi_q\), each summand contributes
\(-\pi_q(\theta)(\nabla_\theta\bar\ell(q,y_q^+;\theta)-\nabla_\theta\bar\ell(q,y_q^-;\theta))\). Averaging over
\(\mathcal Q\) gives
\(\nabla_\theta\mathcal K(\theta)=-g_{\mathrm{KL}}(\theta)\) by Eq.~\eqref{eq:method-gkl}.
\end{proof}

\subsection{Proof of the optimal local direction}
\label{app:proof-optimal-direction}

\begin{proof}
By Eq.~\eqref{eq:method-gkl}, the first-order alignment with the
preference-aware target direction is
{\small
\begin{equation}
-\nabla_\theta\mathcal K(\theta)^\top\delta\theta
=
g_{\mathrm{KL}}(\theta)^\top\delta\theta .
\end{equation}
}
Thus the constrained problem is to maximize \(g_{\mathrm{KL}}(\theta)^\top\delta\theta\)
subject to \(\delta\theta^\top H\delta\theta\le c\). By
Cauchy--Schwarz in the \(H^{-1}\) geometry,
{\small
\begin{equation}
\begin{aligned}
(g_{\mathrm{KL}}(\theta)^\top\delta\theta)^2
&=
\bigl((H^{-1/2}g_{\mathrm{KL}}(\theta))^\top
(H^{1/2}\delta\theta)\bigr)^2\\
&\le
\left(g_{\mathrm{KL}}(\theta)^\top H^{-1}g_{\mathrm{KL}}(\theta)\right)\\
&\qquad\cdot
\left(\delta\theta^\top H\delta\theta\right)\\
&\le
c\,g_{\mathrm{KL}}(\theta)^\top H^{-1}g_{\mathrm{KL}}(\theta).
\end{aligned}
\end{equation}
}
Equality holds iff the Cauchy--Schwarz vectors are proportional, equivalently
{\small
\begin{equation}
\delta\theta\propto H^{-1}g_{\mathrm{KL}}(\theta).
\end{equation}
}
Normalizing this direction to saturate the budget gives
{\small
\begin{equation}
\delta\theta^\star
=
\sqrt{
\frac{c}{
g_{\mathrm{KL}}(\theta)^\top H^{-1} g_{\mathrm{KL}}(\theta)
}}
\,H^{-1}g_{\mathrm{KL}}(\theta).
\label{eq:theory-optimal-direction}
\end{equation}
}
The resulting first-order gain is
{\small
\begin{equation}
-\nabla_\theta\mathcal K(\theta)^\top\delta\theta^\star
=
\sqrt{c}\,\|g_{\mathrm{KL}}\|_{H^{-1}}.
\label{eq:theory-prism-gain}
\end{equation}
}
\end{proof}

\subsection{Proof of the directional gap}
\label{app:proof-cosine-gap}

\begin{proof}
Let \(\langle u,v\rangle_{H^{-1}}=u^\top H^{-1}v\) and \(\|u\|_{H^{-1}}=\sqrt{\langle u,u\rangle_{H^{-1}}}\). To use the same local budget as PRISM, the equal-aggregation update is
{\small
\begin{equation}
\delta\theta_{\mathrm{eq}}
=
\frac{\sqrt{c}}{\|g_0\|_{H^{-1}}}\,
H^{-1}g_0.
\label{eq:theory-delta0}
\end{equation}
}
Using Eq.~\eqref{eq:method-gkl} and Eq.~\eqref{eq:theory-delta0},
{\small
\begin{equation}
\begin{aligned}
-\nabla_\theta\mathcal K(\theta)^\top\delta\theta_{\mathrm{eq}}
&=
g_{\mathrm{KL}}(\theta)^\top\delta\theta_{\mathrm{eq}}
\\
&=
\sqrt c\,
\frac{\langle g_{\mathrm{KL}},g_0\rangle_{H^{-1}}}
{\|g_0\|_{H^{-1}}}\\
&\le
\sqrt c\,\|g_{\mathrm{KL}}\|_{H^{-1}}
\\
&=
-\nabla_\theta\mathcal K(\theta)^\top\delta\theta^\star.
\end{aligned}
\label{eq:theory-cosine-gap}
\end{equation}
}
The optimal value from Eq.~\eqref{eq:theory-optimal-direction} is \(\sqrt c\,\|g_{\mathrm{KL}}\|_{H^{-1}}\). By Cauchy--Schwarz in the \(H^{-1}\) metric,
{\small
\begin{equation}
\langle g_{\mathrm{KL}},g_0\rangle_{H^{-1}}
\le
\|g_{\mathrm{KL}}\|_{H^{-1}}\|g_0\|_{H^{-1}} .
\end{equation}
}
Equivalently, the same argument can be written as the fraction of PRISM's
optimal first-order reward gain retained by equal aggregation:
{\small
\begin{equation}
\begin{aligned}
\frac{
-\nabla_\theta\mathcal K(\theta)^\top\delta\theta_{\mathrm{eq}}
}{
-\nabla_\theta\mathcal K(\theta)^\top\delta\theta^\star
}
&=
\cos_{H^{-1}}\!\left(g_{\mathrm{KL}}(\theta),g_0(\theta)\right)
\\
&=
\frac{\langle g_{\mathrm{KL}},g_0\rangle_{H^{-1}}}
{\|g_{\mathrm{KL}}\|_{H^{-1}}\|g_0\|_{H^{-1}}}
\\
&\le 1.
\end{aligned}
\label{eq:theory-cosine-ratio}
\end{equation}
}
This proves the inequality in Eq.~\eqref{eq:theory-direction-gap-main}. Equality holds
exactly when \(g_0(\theta)\) and \(g_{\mathrm{KL}}(\theta)\) lie on the same positive ray under
the \(H^{-1}\) inner product.
\end{proof}

\subsection{Proof of the equal-budget ranking comparison}
\label{app:proof-equal-budget}

\begin{proof}
Applying Eq.~\eqref{eq:theory-remove-derivative} to \(S=S_\pi^{(m)}\) and \(S=S_0^{(m)}\), then subtracting the two first-order expansions, gives
Eq.~\eqref{eq:theory-gap-main}.

The coefficient in Eq.~\eqref{eq:theory-gap-main} is
{\small
\begin{equation}
\begin{aligned}
\Delta_m
&=
\frac{1}{m}
\left(
\sum_{z_i\in S_\pi^{(m)}}h_\pi(z_i;\theta)
-
\sum_{z_i\in S_0^{(m)}}h_\pi(z_i;\theta)
\right).
\end{aligned}
\label{eq:theory-gap-delta}
\end{equation}
}
It remains only to show \(\Delta_m\ge0\). By construction,
\(S_\pi^{(m)}\) maximizes
\(\sum_{z\in S}h_\pi(z;\theta)\) among all size-\(m\) subsets of
\(\mathcal D\). Since \(S_0^{(m)}\) is one such subset,
{\small
\begin{equation}
\sum_{z\in S_\pi^{(m)}}h_\pi(z;\theta)
\ge
\sum_{z\in S_0^{(m)}}h_\pi(z;\theta),
\end{equation}
}
which proves Eq.~\eqref{eq:theory-gap-delta}. If the inequality is strict,
then \(\Delta_m>0\), and Eq.~\eqref{eq:theory-gap-main} gives a strictly
larger first-order paired KL effect for the preference-weighted subset.
\end{proof}

\section{An Information-Theoretic View of Preference Weighting}
\label{app:info-view}

This appendix gives a complementary interpretation of the preference-weighted
target signal. It is independent of the main exposition: the goal is only to
show why a pair with stronger target-positive preference carries more
information about the target behavior than a low-preference pair.

For a paired target \(q\), define the paired log-odds
{\small
\begin{equation}
r_q(\theta)
=
\log \bar p_\theta(y_q^+\mid q)
-
\log \bar p_\theta(y_q^-\mid q).
\end{equation}
}
Then the paired target-positive preference is
{\small
\begin{equation}
\pi_q(\theta)
=
\frac{\bar p_\theta(y_q^+\mid q)}
{\bar p_\theta(y_q^+\mid q)+\bar p_\theta(y_q^-\mid q)}
=
\sigma(r_q(\theta)).
\end{equation}
}
The target-negative reference distribution assigns all mass to the
target-negative response. The information required to rule out this
target-negative reference under the model's paired distribution is the paired
KL risk
{\small
\begin{equation}
\begin{aligned}
I_q(\theta)
&=
D_{\mathrm{KL}}(A_q\|P_q^\theta)\\
&=
-\log(1-\pi_q(\theta))\\
&=
\log(1+\exp r_q(\theta)).
\end{aligned}
\end{equation}
}
Thus \(I_q\) is the self-information of the target-negative response under the
model's paired distribution. Its marginal sensitivity to the target-positive
paired log-odds is
{\small
\begin{equation}
\frac{\partial I_q(\theta)}{\partial r_q(\theta)}
=
\pi_q(\theta).
\end{equation}
}
This identity gives an information-theoretic meaning to the weight:
increasing the target-positive log-odds by the same small amount changes the
paired KL information more for high-preference pairs than for low-preference pairs.

Using
{\small
\begin{equation}
\begin{aligned}
\nabla_\theta r_q(\theta)
&=
\nabla_\theta \log \bar p_\theta(y_q^+\mid q)
-
\nabla_\theta \log \bar p_\theta(y_q^-\mid q)\\
&=
-\nabla_\theta\bar\ell(q,y_q^+;\theta)
+
\nabla_\theta\bar\ell(q,y_q^-;\theta),
\end{aligned}
\end{equation}
}
we obtain
{\small
\begin{equation}
\begin{aligned}
-\nabla_\theta I_q(\theta)
&=
\pi_q(\theta)\nabla_\theta\bar\ell(q,y_q^+;\theta)\\
&\quad
-
\pi_q(\theta)\nabla_\theta\bar\ell(q,y_q^-;\theta).
\end{aligned}
\end{equation}
}
Averaging over target pairs recovers the preference-weighted target direction in
Eq.~\eqref{eq:method-gkl}. Hence the weighted direction is the negative
gradient of the paired information risk.

The same view also compares the preference-weighted and equal-weight characterizations of
the target behavior. Let
{\small
\begin{equation}
d_q(\theta)
=
\nabla_\theta\bar\ell(q,y_q^+;\theta)
-
\nabla_\theta\bar\ell(q,y_q^-;\theta),
\end{equation}
}
and define
{\small
\begin{equation}
\begin{aligned}
g_{\mathrm{KL}}(\theta)
&=
\frac{1}{|\mathcal Q|}
\sum_{q\in\mathcal Q}\pi_q(\theta)d_q(\theta),\\
\quad
g_0(\theta)
&=
\frac{1}{|\mathcal Q|}
\sum_{q\in\mathcal Q}d_q(\theta).
\end{aligned}
\end{equation}
}
For compactness in the following local-geometry comparison, write
\(a_\theta=g_{\mathrm{KL}}(\theta)\) and \(b_\theta=g_0(\theta)\).
For any local update \(\delta\theta\), its first-order alignment with the
negative gradient of the average paired information risk is
{\small
\begin{equation}
\Delta I(\delta\theta)
=
a_\theta^\top\delta\theta+o(\|\delta\theta\|).
\end{equation}
}
Let \(\langle u,v\rangle_{H^{-1}}=u^\top H^{-1}v\) and
\(\|u\|_{H^{-1}}=\sqrt{\langle u,u\rangle_{H^{-1}}}\). Under the same local budget
\(\delta\theta^\top H\delta\theta\le c\), the maximum first-order
alignment is
{\small
\begin{equation}
\max_{\delta\theta^\top H\delta\theta\le c}
a_\theta^\top\delta\theta
=
\sqrt c\,\|a_\theta\|_{H^{-1}},
\end{equation}
}
achieved by \(\delta\theta\propto H^{-1}a_\theta\). If instead one uses
the equally weighted target direction \(g_0\) and normalizes it to the same
budget,
{\small
\begin{equation}
\delta\theta_{\mathrm{eq}}
=
\frac{\sqrt c}{\|b_\theta\|_{H^{-1}}}\,
H^{-1}b_\theta,
\end{equation}
}
then its first-order alignment is
{\small
\begin{equation}
a_\theta^\top\delta\theta_{\mathrm{eq}}
=
\sqrt c\,
\frac{\langle a_\theta,b_\theta\rangle_{H^{-1}}}
{\|b_\theta\|_{H^{-1}}}
\le
\sqrt c\,\|a_\theta\|_{H^{-1}}.
\end{equation}
}
Equivalently, the ratio between the alignment of the equal-weight direction and
the optimal weighted direction is
{\small
\begin{equation}
\frac{a_\theta^\top\delta\theta_{\mathrm{eq}}}
{\sqrt c\,\|a_\theta\|_{H^{-1}}}
=
\cos_{H^{-1}}(a_\theta,b_\theta)
\le 1 .
\end{equation}
}
Equality holds only when \(g_0(\theta)\) and \(g_{\mathrm{KL}}(\theta)\) are aligned in the
\(H^{-1}\) metric. Therefore, unless all target pairs have effectively
the same preference weight or their gradient contrasts point in the same
direction, equal aggregation gives a less faithful first-order
characterization of the target behavior than the preference-weighted
information-risk direction.

\end{document}